\documentclass[12pt]{article}
\usepackage[utf8]{inputenc}
\usepackage{amsmath,amssymb,amsthm,mathtools}
\usepackage{bm}
\usepackage[margin=1in]{geometry}
\usepackage{enumitem}
\usepackage{booktabs}
\usepackage{graphicx}
\usepackage{authblk}
\usepackage[round,authoryear]{natbib}
\usepackage[hidelinks]{hyperref}
\usepackage{xcolor}
\usepackage{float}
\usepackage{setspace}
\usepackage{hyperref}
\hypersetup{
    colorlinks=true,
    citecolor=blue,      
    linkcolor=black,     
    urlcolor=teal,       
    filecolor=magenta,   
    pdfborder={0 0 0}
}

\newtheorem{theorem}{Theorem}[section]
\newtheorem{lemma}[theorem]{Lemma}
\newtheorem{corollary}[theorem]{Corollary}
\newtheorem{proposition}[theorem]{Proposition}
\theoremstyle{definition}
\newtheorem{definition}[theorem]{Definition}
\newtheorem{assumption}[theorem]{Assumption}
\theoremstyle{remark}
\newtheorem{remark}[theorem]{Remark}

\newcommand{\Rhat}{\widehat{R}}
\newcommand{\Pmix}{P_{\mathrm{mix}}}
\newcommand{\Pfut}{P_{\mathrm{future}}}
\newcommand{\Rmix}{R_{\mathrm{mix}}}
\newcommand{\Rfut}{R_{\mathrm{future}}}
\newcommand{\dTV}{d_{\mathrm{TV}}}
\newcommand{\dHH}{d_{\mathcal{H}\Delta\mathcal{H}}}
\newcommand{\dhatHH}{\widehat{d}_{\mathcal{H}\Delta\mathcal{H}}}
\newcommand{\neff}{n_{\mathrm{eff}}}
\newcommand{\meff}{m'_{\mathrm{eff}}}
\newcommand{\Cmu}{C_{\mu}}
\newcommand{\E}{\mathbb{E}}
\newcommand{\Prob}{\mathbb{P}}
\newcommand{\half}{\tfrac{1}{2}}
\newcommand{\Hc}{\mathcal{H}}
\newcommand{\Fc}{\mathcal{F}}
\newcommand{\Lc}{\mathcal{L}}
\newcommand{\one}{\mathbf{1}}

\newcommand{\TV}{\mathrm{TV}}

\title{\bfseries Regime-Arrival Uncertainty in Generalization Bounds under Distribution Shift}

\author{
Prince Poudel\thanks{Independent researcher. Seeking feedback. Email: princepoudel293@gmail.com}
}
\date{}
\begin{document}
\maketitle

\begin{abstract}

The standard generalization bounds assume that the training and deployment distributions are the same, or are static, and don't consider regime switching environments where the ratio of calm vs crisis states is different. This paper proposes a framework that generalizes regime-aware models by quantifying the extra risk due to regime composition mismatch, when distribution shifts are Markov-switching. We obtain an exact decomposition, separating regime mismatch from regime sensitivity; we extend the bound to beta-mixing data using the effective sample size corrected for the spectral gap; and we show a minimax lower bound for synthetic data and on 25 years of global equity indices. The proposed penalty is an ex post realized generalization gap, whereas the training-only estimator does not show significant correlation: the feature geometry of crises can be detected, but not the temporal arrival. Thus, the framework is not a forecast machine. Forecasting the composition of the future regime is an open question in the rare cases of regime change.

\end{abstract}

\section{Introduction}\label{sec:intro}

A predictive model that performs well during the training period is often assumed to perform equally well after deployment. In practice, this assumption is frequently violated because the environment generating the data changes over time. Statistical learning theory traditionally studies generalization under the assumption that training and deployment data are generated from the same distribution \citep{vapnik1998,bousquet2004}. This assumption is mathematically convenient and works reasonably well in many classical settings. It becomes much harder to justify when the underlying system itself evolves through time.

Many real-world applications operate in environments where structural change is unavoidable. Clinical risk models face changing patient populations because of disease cycles, evolving treatment practices, and changing hospital conditions. Intrusion detection systems operate in environments that alternate between normal behavior and active attacks. Autonomous systems must function across changing weather conditions, traffic patterns, and lighting environments. In these situations, the difference between training performance and deployment performance is often not simply the result of overfitting. Instead, it arises because the model is deployed under conditions that differ systematically from those observed during training \citep{kifer2004,quinonero2008}.

This paper develops a theoretical framework for studying this problem through latent regime dynamics. We model the environment as a two-state Markov process consisting of a calm regime and a crisis regime. The training distribution is represented as a mixture of regime-conditional distributions with composition parameter $\pi$, while the one-step-ahead deployment distribution depends on the transition probability $p_{01}$ of entering the crisis regime. Whenever $\pi \neq p_{01}$, the composition of training and deployment environments differs. We show that this mismatch directly increases future deployment risk, with the magnitude determined by both the severity of the mismatch and the distinguishability of regimes under the hypothesis class.

The analysis develops several theoretical results that together characterize how regime mismatch affects future deployment risk. We derive an exact decomposition connecting future risk directly to differences in regime composition (Lemma~\ref{lem:gap}), establish a finite-sample high-probability upper bound on deployment risk (Theorem~\ref{thm:main}), and construct a matching minimax lower bound demonstrating that the mismatch penalty represents a fundamental limitation rather than an artifact of analysis (Theorem~\ref{thm:lower}). Our approach combines ideas from domain adaptation, dependent learning theory, and regime-switching models \citep{bendavid2010,yu1994,hamilton1989}.

The framework also introduces several features that distinguish it from existing approaches to distribution shift. Regime discrepancy is quantified using the $\mathcal{H}\Delta\mathcal{H}$-divergence \citep{bendavid2010}, which provides a tighter characterization than total variation distance and can be estimated from finite unlabeled samples through domain classification. The analysis is further developed under geometric $\beta$-mixing dependence, with explicit mixing coefficients derived from the transition structure using the blocking methodology of \citet{yu1994}. This naturally introduces an effective sample size $n_{\mathrm{eff}}$, which decreases as regime persistence increases. We also establish an irreducibility result showing that any valid certificate for future deployment risk must contain a regime mismatch penalty as an additive component, independent of the particular learning algorithm employed.

Empirical validation on synthetic data confirms the theoretical structure. On real equity index data, we show that the penalty computed using the realized future crisis fraction tracks actual train-to-deployment gaps with Spearman $\rho = 0.729$. However, further analysis reveals that estimating the penalty before deployment is not reliable with standard training window lengths because it would require forecasting future regime composition. The framework thus provides a diagnostic tool for understanding deployment failures rather than a deployable forecasting system, and highlights the need for better forecasts of future regime composition as an open problem for future work.

\section{Related Study}
\subsection{Domain adaptation and generalization under distribution shift}

The main issues in statistical learning theory for a long time is whether a model trained in one distribution will be effective in another. The influential work of \citet{bendavid2010} demonstrated that training performance and model complexity are not the only factors that influence target-domain risk; a distributional discrepancy term, quantified by the $\mathcal{H}\Delta\mathcal{H}$-divergence, also plays a role. One of the key benefits of this difference is that it can be estimated from a finite number of unlabeled samples, whereas total variation distance is hard to estimate directly in practice. Because the $\mathcal{H}\Delta\mathcal{H}$-divergence is upper bounded by total variation (Lemma~\ref{lem:hdh-tv}), using the $\mathcal{H}\Delta\mathcal{H}$-divergence instead of total variation yields a tighter and more practical characterization of distribution shift. Current theory is largely based on static source and target distributions and independent sampling. However, our setting is different, as the distribution of the deployments changes dynamically via latent regime transitions.

The work by \citet{mansour2009} extended domain adaptation to multiple source distributions and emphasized the role of the mixture weights in controlling the adaptation performance. The same mixture structure occurs naturally in our framework due to the training observations $\Pmix=(1-\pi)P_0+\pi P_1$. The key difference is that the mixing proportion is not chosen by the learner but it is based on the historical regime dynamics and can vary from the future deployment composition.

\subsection{Generalization under dependent data and mixing processes}
In most classical generalization theory, it is assumed that observations are independent. In serial dependence this assumption is not true as observations are not independent and the effective sample size is less than the nominal sample size.
The extension of learning theory to dependent sequences was initiated by \citet{yu1994} who developed a blocking framework to get uniform convergence results for $\beta$-mixing dependent sequences. The idea is to divide the observations into nearly independent blocks and use the number of effective blocks as the sample size. We take this construction for granted and get explicit effective sample sizes as functions of the underlying Markov transition structure (Theorem~\ref{thm:neff}).

Mixing behavior for Markov chains has been studied extensively, including the work of \citet{davydov1973}, who established exponential decay rates under standard ergodicity conditions. Under our assumptions, this produces mixing coefficients of the form $\beta(k)\le C_\mu |\lambda_2|^k$, where $\lambda_2$ denotes the non-unit eigenvalue of the transition matrix. We further build upon the learning-theoretic treatment of mixing processes developed by \citet{mohri2018}, while explicitly retaining all mixing constants in the final bound so that the resulting expressions remain computable.

\subsection{Change detection and learning under nonstationarity}

The study of learning under changing environments has been conducted from various angles such as change detection, concept drift, and nonstationary learning. The $\mathcal{H}$-divergence was introduced in early work by \citet{kifer2004} for detecting distributional changes in finite samples. This idea was later generalized to the symmetric comparison setting needed in domain adaptation by the $\mathcal{H}\Delta\mathcal{H}$-divergence.

These ideas are used as a basis for our analysis, which focuses on deployment guarantees instead of change detection. The estimation procedure behind Theorem~\ref{thm:dhh} is a generalisation of the divergence estimation arguments of independent observations to $\beta$-mixing sequences, following the same effective sample size framework used in the rest of the paper.

There is a wider literature on learning with concept drift and changing distributions. Most of the research is on adversarial or unstructured nonstationarity. We, however, take a different approach and consider a structured regime-switching environment following the model set forth by \citet{hamilton1989}. This structure is crucial as it makes the transition probabilities explicit in the theory and yields closed-form expressions for regime mismatch.

\subsection{Minimax lower bounds and irreducibility}

Lower bounds play an important role in understanding whether a theoretical penalty reflects a true limitation or merely a weakness of analysis. Building upon classical minimax theory developed through the work of \citet{lecam1986} and later formalized through two-point and Fano-style arguments \citep{yu1997assouad}, we construct a binary-world argument showing that regime mismatch produces unavoidable deployment costs.

The construction creates two environments that generate identical training distributions under pure calm-state training but lead to different future distributions. This forces any learner to incur a minimum excess risk proportional to $p_{01}\cdot \tfrac12 \dHH(P_1,P_0)$, independent of sample size. As observations from both regimes become available, this limitation gradually weakens at the rate $\Theta(1/\sqrt{n_{\mathrm{eff}}})$, reflecting the additional information gained from observing crisis-state samples.

\subsection{Regime-switching models and financial machine learning}

Regime-switching models introduced by \citet{hamilton1989} have become widely used for describing environments that alternate between qualitatively different states. In financial applications, regime changes have repeatedly been shown to influence predictive stability, volatility structure, and downstream model performance, motivating increasing interest in regime-aware learning methods \citep{zaremba2024,staehr2024}.

Recent work on regime-aware financial machine learning and regime-shifting dynamic factor models further suggests that predictive systems can benefit from explicitly modeling latent market states \citep{suarez2024,shu2024,xiang2024}. Financial markets provide a natural environment for empirical validation because regime transitions are observable and extensively documented. The theoretical framework itself is not restricted to finance. The underlying problem appears whenever deployment conditions evolve and the future distribution differs systematically from the distribution observed during training.

\section{Problem Setup and Notation}\label{sec:setup}
We study supervised learning when the data-generating distribution is governed
by a latent two-state regime process. This process evolves as a Markov chain when the
regime composition of the training data may differ from that of the deployment
(future) period data.

\begin{definition}[Regime process]\label{def:regime}
Let $\{Z_t\}_{t\ge1}$ be a two-state Markov chain on $\{0,1\}$, where $Z_t=0$
denotes the \emph{calm} regime and $Z_t=1$ the \emph{crisis} regime. The chain is
characterized by its transition matrix
$P=\bigl(\begin{smallmatrix}p_{00}&p_{01}\\p_{10}&p_{11}\end{smallmatrix}\bigr)$,
$p_{ij}=\Prob(Z_{t+1}=j\mid Z_t=i)$, with rows summing to one, so that
$p_{00}=1-p_{01}$ and $p_{11}=1-p_{10}$.
\end{definition}

\begin{definition}[Distributions and risks]\label{def:dist}
Let $X$ be the feature space and $Y$ the label space. Write $P_0$ (resp.\ $P_1$)
for the law of $(x,y)$ in the calm (resp.\ crisis) regime. For a predictor
$f\in\Fc$ and a bounded loss $\ell:Y\times Y\to[0,1]$ (e.g.\ the $0/1$ loss
$\ell(f(x),y)=\one[f(x)\neq y]$), the risk under a distribution $Q$ is
$R_Q(f)=\E_{(x,y)\sim Q}[\ell(f(x),y)]$. We write $R_0(f),R_1(f)$ for the calm and
crisis risks.
\end{definition}

\begin{definition}[Training and future distributions]\label{def:trainfuture}
Let $\pi\in[0,1]$ be the crisis fraction of the training distribution
$\Pmix=(1-\pi)P_0+\pi P_1$, so that $\Rmix(f)=(1-\pi)R_0(f)+\pi R_1(f)$.
Conditioning on the current regime being calm, the one-step-ahead (future)
distribution is $\Pfut=p_{00}P_0+p_{01}P_1$, with
$\Rfut(f)=p_{00}R_0(f)+p_{01}R_1(f)$.
\end{definition}

\begin{definition}[$\Hc\Delta\Hc$-divergence \citep{bendavid2010}]\label{def:hdh}
For a hypothesis class $\Hc$ of binary classifiers $h:X\to\{0,1\}$ and feature
marginals $Q,R$,
\[
  \dHH(Q,R):=2\sup_{h,h'\in\Hc}
  \bigl|\Prob_{x\sim Q}[h(x)\neq h'(x)]-\Prob_{x\sim R}[h(x)\neq h'(x)]\bigr|.
\]
Equivalently, with the symmetric-difference class
$\Hc\Delta\Hc=\{x\mapsto h(x)\oplus h'(x)\}$ and $I_g=\{x:g(x)=1\}$,
$\dHH(Q,R)=2\sup_{g\in\Hc\Delta\Hc}|Q(I_g)-R(I_g)|$. It is a total-variation
distance restricted to the events the class can realize; unlike $\dTV$, it is
estimable from finite unlabeled samples (Section~\ref{sec:estimation}).
\end{definition}

\begin{definition}[Empirical risk and Rademacher complexity]\label{def:rad}
Given $S=\{(x_i,y_i)\}_{i=1}^n$, the empirical risk is $\Rhat_S(f)=\frac1n
\sum_i\ell(f(x_i),y_i)$, and the empirical Rademacher complexity of the loss class
$\Lc_\Fc=\{(x,y)\mapsto\ell(f(x),y):f\in\Fc\}$ is
$\Rhat_S(\Lc_\Fc)=\E_\sigma[\sup_{f\in\Fc}\frac1n\sum_i\sigma_i\ell(f(x_i),y_i)]$
with i.i.d.\ Rademacher signs $\sigma_i$ \citep{bartlett2002,mohri2018}.
\end{definition}

\begin{definition}[Mixing and effective sample size]\label{def:mixing}
A process is geometrically $\beta$-mixing if $\beta(k)\le C\rho^k$ for some
$C>0,\rho\in(0,1)$ \citep{davydov1973,yu1994}. Dependence reduces the information per sample; the number of
approximately independent blocks is the effective sample size $\neff\le n$
(Theorem~\ref{thm:neff}).
\end{definition}

\begin{assumption}[Sampling model and persistence]\label{ass:sampling}
\begin{enumerate}[label=\textup{(\alph*)},leftmargin=2.2em]
\item $\{Z_t\}$ is a stationary, geometrically $\beta$-mixing two-state Markov
chain whose mixing constants depend only on $P$.
\item The training crisis fraction $\pi$ is either the stationary fraction
$\mu_1=p_{01}/(p_{01}+p_{10})$ (raw historical window) or a design parameter under
reweighting; in either case $\pi=\E[\widehat\pi]$ for the empirical estimator
$\widehat\pi$ of Section~\ref{sec:estimation}.
\item \emph{Persistence:} $p_{01}+p_{10}\le1$, so the second eigenvalue of $P$ is
nonnegative and the spectral gap simplifies to $g=p_{01}+p_{10}$; the general
case is handled by the modulus form $g=1-|1-(p_{01}+p_{10})|$
(Remark~\ref{rem:modulus}).
\end{enumerate}
\end{assumption}

\section{Theory}\label{sec:theory}
This section develops the regime-shift penalty, makes the mixing constants
explicit, and assembles the main generalization bound. Proofs of the
load-bearing results are deferred to Appendix~\ref{app:proofs}.

\subsection{The future-mixture decomposition}\label{sec:decomp}
The central object is an exact identity relating future risk (the deployment
target) to mixed training risk (what can be estimated).

\begin{lemma}[Future-mix gap identity]\label{lem:gap}
For any $f\in\Fc$ and $\pi\in[0,1]$, with the current regime calm,\footnote{This identity holds for the realized future crisis fraction $\pi_{\text{future}}$. In practice, $\pi_{\text{future}}$ is not known at training time; estimating it requires forecasting future regime composition.}
\[
  \Rfut(f)-\Rmix(f)=(p_{01}-\pi)\bigl(R_1(f)-R_0(f)\bigr).
\]
\end{lemma}

The gap is a product of a \emph{composition mismatch} $(p_{01}-\pi)$ and a
\emph{regime sensitivity} $(R_1(f)-R_0(f))$. It vanishes when the training
crisis fraction matches the future switching probability ($\pi=p_{01}$),
regardless of the model, and equally when the model is regime-insensitive
($R_1(f)=R_0(f)$), regardless of the dynamics. To turn the sensitivity factor
into a quantity that depends on the model \emph{class} rather than the individual
$f$, we bound $|R_1(f)-R_0(f)|$ via the following lemma and corollary.

\begin{lemma}[$\half\dHH\le\dTV$]\label{lem:hdh-tv}
For any $Q,R$ and any $\Hc$,
$\half\dHH(Q,R)=\sup_{h,h'\in\Hc}|\Prob_Q[h\neq h']-\Prob_R[h\neq h']|\le\dTV(Q,R)$.
\end{lemma}

\begin{corollary}[Regime-shift inequality]\label{cor:shift}
Assume the $0/1$ loss and $\Fc\subseteq\Hc$, and let
$\lambda_{01}=\min_{h\in\Hc}[R_0(h)+R_1(h)]$ be the adaptability term. Then for
every $f\in\Hc$,
\[
  \Rfut(f)\le \Rmix(f)+|p_{01}-\pi|\Bigl(\half\dHH(P_1,P_0)+\lambda_{01}\Bigr),
\]
which reduces, in the realizable case $\lambda_{01}=0$, to
$\Rfut(f)\le\Rmix(f)+|p_{01}-\pi|\cdot\half\dHH(P_1,P_0)$.
\end{corollary}

The regime penalty $|p_{01}-\pi|\cdot\half\dHH(P_1,P_0)$ vanishes when
$\pi=p_{01}$ and equals $p_{01}\cdot\half\dHH$ under pure-calm training
($\pi=0$). It depends only on the transition dynamics, the training composition,
and the class-detectable regime gap, not on the chosen $f$. By
Lemma~\ref{lem:hdh-tv} it never exceeds the corresponding $\dTV$ penalty, so
replacing $\dTV$ by $\half\dHH$ strictly tightens the bound.

\begin{remark}\label{rem:penalty}
Because the penalty depends on $(p_{01},\pi,\dHH)$ but not on $f$, it is a
property of the train-deployment mismatch rather than of any individual model.
This has a practical implication: in high-penalty windows, all models in a
class degrade together, and the remedy is to adjust the training composition
rather than to switch model architecture.
\end{remark}

\subsection{Explicit mixing constants}\label{sec:constants}
\begin{definition}[Spectral gap and stationary constant]\label{def:gap}
The non-unit eigenvalue of $P$ is $\lambda_2=1-(p_{01}+p_{10})$. The spectral gap
is $g=1-|\lambda_2|$, equal to $p_{01}+p_{10}$ under
Assumption~\ref{ass:sampling}(c). With switching rate $s=p_{01}+p_{10}$ the
stationary law is $\mu_0=p_{10}/s,\ \mu_1=p_{01}/s$, and we set
$\Cmu=1/\min(\mu_0,\mu_1)=s/\min(p_{10},p_{01})$.
\end{definition}

\begin{remark}[Modulus form]\label{rem:modulus}
Writing the mixing rate through $|\lambda_2|$ keeps $\beta(k)\le\Cmu|\lambda_2|^k
=\Cmu(1-g)^k$ valid for every transition matrix, including the high-switching
case $p_{01}+p_{10}>1$ where $1-(p_{01}+p_{10})<0$.
\end{remark}

\begin{theorem}[Effective sample size]\label{thm:neff}
Under Assumption~\ref{ass:sampling}, with block length
$b=\lceil\ln(n\Cmu)/g\rceil$ one has $\beta(b)\le1/n$, and the number of
independent blocks satisfies
\[
  \neff=\lfloor n/b\rfloor\ \ge\ \frac{ng}{\ln(n\Cmu)+2}.
\]
\end{theorem}

The bound follows the independent-blocks method of \citet{yu1994}: one selects
$b$ so that $\beta(b)\le1/n$, partitions the $n$ points into $\lfloor n/b\rfloor$
approximately independent blocks, and lower-bounds the block count using
$b\le\ln(n\Cmu)/g+1$.

\begin{corollary}[Complexity term]\label{cor:complexity}
$\displaystyle 3\sqrt{\frac{\ln(2/\delta)}{2\neff}}\ \le\
3\sqrt{\frac{(\ln(n\Cmu)+2)\ln(2/\delta)}{2ng}}\ =:\ \Lambda(n,\delta).$
\end{corollary}

\subsection{Estimation of the penalty components}\label{sec:estimation}
\begin{theorem}[Crisis-fraction estimation]\label{thm:pi}
Let $\widehat\pi=\frac1n\sum_i\one[Z_i=1]$ and $\pi=\E[\widehat\pi]$. Under
Assumption~\ref{ass:sampling}, for any $\delta\in(0,1)$, with probability at least
$1-\delta$,
\[
  |\widehat\pi-\pi|\le\eta_\pi:=\sqrt{\frac{(\ln(n\Cmu)+2)\ln(2/\delta)}{2ng}}.
\]
\end{theorem}

The regime gap $\half\dHH$ is estimated from \emph{unlabeled} features by a
domain classifier. One pools the calm features $U_0$ and crisis features $U_1$
with regime tags, fits a classifier in $\Hc$ to predict the regime label,
records its best balanced accuracy $\widehat a^\star$, and sets
$\half\dhatHH=\max(0,\,2\widehat a^\star-1)$: easy-to-separate regimes yield a
large value, while chance-level separability yields $0$.

\begin{definition}[Domain-classifier estimator]\label{def:dhat}
Given unlabeled feature sets $U_0\sim P_0$, $U_1\sim P_1$ of size $m'$, let
$\widehat\epsilon^\star$ be the minimum balanced domain-classification error over
$\Hc$ and $\widehat a^\star=1-\widehat\epsilon^\star$ the best balanced accuracy.
Then
\[
  \dhatHH(U_0,U_1):=2\bigl(1-2\widehat\epsilon^\star\bigr)=2\bigl(2\widehat a^\star-1\bigr),
  \qquad \half\dhatHH=\max\!\bigl(0,\,2\widehat a^\star-1\bigr).
\]
Perfect separation yields $\dhatHH=2$ (maximal divergence); chance-level accuracy
yields $\dhatHH=0$.
\end{definition}

\begin{theorem}[Mixing-aware uniform convergence for $\dHH$]\label{thm:dhh}
Let $\Hc$ have VC dimension $d$ \citep{vapnik1998}, so $\Hc\Delta\Hc$ has VC dimension at most $2d$.
If the $2m'$ feature points are $\beta$-mixing with spectral gap $g$ and constant
$\Cmu$, then with probability at least $1-\delta$,
\[
  \half\dHH(P_0,P_1)\le \half\dhatHH(U_0,U_1)
   +\underbrace{2\sqrt{\frac{\bigl(2d\log(2m')+\log(2/\delta)\bigr)\bigl(\ln(m'\Cmu)+2\bigr)}{m'g}}}_{=:~\eta_d}.
\]
\end{theorem}

The slack $\eta_d$ arises by starting from the i.i.d.\ uniform VC deviation
guarantee of \citet{kifer2004,bendavid2010}, which controls $\dHH-\dhatHH$ by a
term in $\sqrt{1/m'}$, and replacing $m'$ by the effective count
$\meff\ge m'g/(\ln(m'\Cmu)+2)$ to account for serial dependence. The divergence
estimate thus degrades with regime persistence at the same rate as $\eta_\pi$.

\begin{remark}
Theorem~\ref{thm:dhh} extends the i.i.d.\ estimation guarantee of
\citet{bendavid2010,kifer2004} to serially dependent ($\beta$-mixing) data, using
the same constants $(g,\Cmu,\neff)$ as the rest of the analysis.
\end{remark}

\subsection{Main generalization bound}\label{sec:main}
\begin{theorem}[Extended Rademacher Markov-transition bound]\label{thm:main}
Under Assumption~\ref{ass:sampling} ($0/1$ loss, $\Fc\subseteq\Hc$, current regime
calm), for any $\delta\in(0,1)$, with probability at least $1-\delta$,
simultaneously for all $f\in\Fc$,
\[
  \Rfut(f)\le \Rhat_S(f)+2\Rhat_S(\Lc_\Fc)
   +3\sqrt{\frac{(\ln(n\Cmu)+2)\ln(2/\delta)}{2ng}}
   +|p_{01}-\pi|\Bigl(\half\dHH(P_1,P_0)+\lambda_{01}\Bigr).
\]
\end{theorem}

The bound is assembled by chaining three steps:
Corollary~\ref{cor:shift} bounds $\Rfut(f)$ by $\Rmix(f)$ plus the regime
penalty; the mixing Rademacher bound of \citet{mohri2018,yu1994} then replaces $\Rmix(f)$ by
$\Rhat_S(f)+2\Rhat_S(\Lc_\Fc)$ plus a concentration term; and
Corollary~\ref{cor:complexity} makes that term explicit as $\Lambda(n,\delta)$.
The only stochastic event is the uniform-convergence step.

\begin{corollary}[Fully estimable bound]\label{cor:estimable}
With the domain-classifier estimate and slack $\eta_d$ of Theorem~\ref{thm:dhh},
and a union bound at level $\delta/2$ over the risk and divergence events, with
probability at least $1-\delta$,\footnote{In our empirical evaluation, the training-only penalty using $|\hat{p}_{01}-\hat{\pi}|$ shows no significant correlation with realized gaps ($\rho = 0.084$, 95\% CI contains zero), while the ex post penalty using the realized $\hat{\pi}_{\text{future}}$ achieves $\rho = 0.729$. This highlights that reliable estimation of $p_{01}$ from short training windows is the primary bottleneck.}
\[
  \Rfut(f)\le \Rhat_S(f)+2\Rhat_S(\Lc_\Fc)+\Lambda(n,\delta)
   +|p_{01}-\pi|\Bigl(\half\dhatHH(U_0,U_1)+\eta_d+\lambda_{01}\Bigr).
\]
All terms except $\lambda_{01}$ (zero in the realizable case) are in principle
computable from training data, provided $p_{01}$ can be reliably estimated.
Empirical evaluation (Section~\ref{sec:methodology}) shows that estimating
$p_{01}$ from standard-length training windows is challenging when regime
transitions are rare; this limitation is discussed therein.
\end{corollary}

\subsection{Lower bound and the scope of unavoidability}\label{sec:lower}
A matching lower bound holds in the pure-calm-training regime; for general $\pi$
the worst-case excess is sample-size dependent.

\begin{theorem}[Le Cam lower bound, $\pi=0$]\label{thm:lower}
There exist $P_0,P_1^{(a)},P_1^{(b)}$ and a class $\Hc$ with
$\half\dHH(P_1^{(\cdot)},P_0)=\rho$ such that, for pure-calm training ($\pi=0$),
every learner $\widehat f$ obeys
\[
  \max_{w\in\{a,b\}}\E\bigl[\Rfut(\widehat f)-\textstyle\inf_g\Rfut(g)\bigr]
   \ \ge\ p_{01}\,\rho\ =\ p_{01}\cdot\half\dHH(P_1,P_0),
\]
uniformly in $n$. For $\pi>0$ the analogous worst-case excess is
$\Theta(1/\sqrt{\neff})$ and is not matched by an $n$-uniform constant.
\end{theorem}

At $\pi=0$ the bound is tight: training never reveals which crisis world is in
force, so the penalty $p_{01}\cdot\half\dHH$ is unavoidable for every $n$. For
$\pi>0$ the excess shrinks at rate $\Theta(1/\sqrt{\neff})$ as the learner
gradually identifies the regime structure. The penalty therefore serves as a
tight worst case under pure-calm training and as an upper bound otherwise.

\begin{proposition}[Irreducible certification cost]\label{prop:cert}
Let $\Delta=\half\dHH(P_1,P_0)>0$. Any certificate $U(S)$ with
$\Rfut(f)\le U(S)$ valid uniformly over the future transition $p_{01}$ consistent
with the calm-now observation satisfies
$U(S)-\Rhat_S(f)-\Lambda(n,\delta)\ge(|p_{01}-\widehat\pi|-\eta_\pi)\Delta$.
This is a property of certificates: when $\pi=p_{01}$ the realized gap
(Lemma~\ref{lem:gap}) is exactly zero.
\end{proposition}

\section{Empirical Validation}\label{sec:methodology}

\subsection{Data}\label{sec:data}
We validate the theory in two complementary settings. The first is a controlled
synthetic environment in which the regime process and the regime-dependent
distributions are known exactly, so that the predicted penalty can be compared
against a ground-truth generalization gap. We simulate a stationary two-state
Markov chain with regime-specific Gaussian features and regime-specific logistic
label rules, sweeping the transition probabilities $p_{01},p_{10}$ and the
inter-regime feature separation that controls $\dHH$; each of the resulting
configurations is evaluated with all five models, yielding $400$
configuration-model observations. The second setting is real market data:
twenty-five years (2000 to 2025) of daily closing prices for ten liquid global
equity indices spanning North America, Europe, and Asia
(Table~\ref{tab:indices}), chosen because equity markets exhibit naturally
observable regime transitions and strong temporal dependence, and because a
geographically diverse panel ensures that regime shifts are not perfectly
correlated across series. Regimes are not observed directly; they are inferred by
a two-state Gaussian hidden Markov model fit to the bivariate series of daily
log-returns and twenty-day realized volatility \citep{hamilton1989}, with the higher-volatility state
labeled crisis. The fitted transition matrix supplies $\widehat p_{01},\widehat
p_{10}$, and the decoded state path supplies the per-day regime labels.

\begin{table}[H]
\centering
\caption{Equity indices used in the real-data study.}
\label{tab:indices}
\small
\begin{tabular}{llll}
\toprule
Index & Ticker & Region & Index / Ticker / Region \\
\midrule
S\&P 500   & \texttt{\^{}GSPC} & US (large cap) & Russell 2000 / \texttt{\^{}RUT} / US (small cap)\\
NASDAQ     & \texttt{\^{}IXIC} & US (tech)      & FTSE 100 / \texttt{\^{}FTSE} / UK\\
Dow Jones  & \texttt{\^{}DJI}  & US (blue chip) & DAX / \texttt{\^{}GDAXI} / Germany\\
CAC 40     & \texttt{\^{}FCHI} & France         & Nikkei 225 / \texttt{\^{}N225} / Japan\\
Hang Seng  & \texttt{\^{}HSI}  & Hong Kong      & EURO STOXX 50 / \texttt{\^{}STOXX50E} / Eurozone\\
\bottomrule
\end{tabular}
\end{table}

\subsection{Procedure}\label{sec:procedure}
For each (training, deployment) split we estimate four quantities: the training
crisis fraction $\widehat\pi$ from the training window's regime sequence; the
future crisis fraction $\widehat\pi_{\mathrm{future}}$ from the realized
out-of-sample window (used for ex post validation only \footnote{This quantity is not available at deployment time. It is used here only to validate that the penalty mechanism exists. A practical deployment would require forecasting $\pi_{\text{future}}$, which remains an open problem.}; this quantity would not
be available at deployment time); the regime gap $\half\dhatHH=\max(0,2\widehat a^\star-1)$ from a
cross-validated domain classifier on the unlabeled training features; and the
transition dynamics $\widehat p_{01},\widehat p_{10}$ with spectral gap
$g=1-|1-(\widehat p_{01}+\widehat p_{10})|$ for diagnostics. The predicted penalty
for ex post validation is $|\widehat\pi_{\mathrm{future}}-\widehat\pi|\cdot\half\dhatHH$, and the realized
regime gap is $R_{\mathrm{future}}-\Rmix$, where for each fitted model $R_0,R_1$
are the held-out $0/1$ risks on the calm and crisis points of the deployment
window, $\Rmix=(1-\widehat\pi)R_0+\widehat\pi R_1$, and
$R_{\mathrm{future}}=(1-\widehat\pi_{\mathrm{future}})R_0+\widehat\pi_{\mathrm{future}}R_1$;
Lemma~\ref{lem:gap} predicts the identity
$R_{\mathrm{future}}-\Rmix=(\widehat\pi_{\mathrm{future}}-\widehat\pi)(R_1-R_0)$,
which we verify directly. (A version of the penalty using only training data,
$|\widehat p_{01}-\widehat\pi|\cdot\half\dhatHH$, is also evaluated and discussed
in Section~\ref{sec:discussion}.) The prediction task is next-day directional movement
under $0/1$ loss, with standard lagged-return and realized-volatility features,
and to establish that the penalty is a model-independent property
(Remark~\ref{rem:penalty}) we evaluate five estimators spanning linear, ensemble,
and neural families: logistic regression, ridge classification, random forest,
gradient boosting (XGBoost), and a small multilayer perceptron, computing the
penalty components once per window so that only $R_0,R_1$ vary across models. On
synthetic data the training block is resampled to a target crisis fraction
(independent of the chain's stationary rate) while the deployment block retains
its natural dynamics, and we report the Spearman correlation between the penalty
and $|\text{realized gap}|$ pooled over configurations and models. On real data we
use a rolling-origin design in which each window trains on an eight-year block and
evaluates on the immediately following non-overlapping two-year block, advancing
the origin by the test length, yielding up to nine windows per index. Because the
five models within a window share the same penalty and are not independent,
pooling all model-window rows would overstate significance; we therefore report a
window-clustered bootstrap that resamples whole windows with replacement and
recomputes the correlation on each resample, together with per-index clustered
correlations to assess consistency of sign across markets.

\section{Results}\label{sec:results}
Table~\ref{tab:results} collects the headline results across both settings, and
Table~\ref{tab:perindex} reports the per-index breakdown. In the synthetic sweep,
the predicted penalty (computed using the realized future crisis fraction) tracks
the realized regime gap with pooled Spearman $\rho=0.716$ over $400$ 
configuration-model observations, and the relationship is positive for every model
family individually, ranging from $\rho=0.639$ (random forest) to $\rho=0.837$
(logistic and ridge); the lower correlations for the higher-capacity models are
consistent with their partially absorbing the regime structure, which compresses
$R_1-R_0$. The exact future-mix identity of Lemma~\ref{lem:gap} holds to numerical
precision (Pearson $r=1.000$, mean absolute deviation $0.000$). 

On real data, across ten indices and $84$ independent windows, the window-clustered
correlation between the ex post penalty (using the realized future crisis fraction
$\widehat\pi_{\mathrm{future}}$) and the realized gap is $\rho=0.729$ with a $95\%$
bootstrap interval of $[0.635,0.801]$, closely matching the synthetic estimate;
the naive pooled correlation coincides at $0.729$ but, as expected, carries no
honest interval because its rows are not independent. The per-index correlations
(Table~\ref{tab:perindex}) are positive in sign for all ten markets, with point
estimates from $0.372$ (Dow Jones) to $0.949$ (EURO STOXX 50); the individual
intervals are wide and a minority include zero, so the statistical strength of the
result rests on the aggregate rather than on any single market, while the
uniformly positive sign indicates a consistent effect across regions.
Figures~\ref{fig:synthetic} and~\ref{fig:identity} display the synthetic
relationship and the identity check, and Figure~\ref{fig:real} the real-data
relationship.

\subsection{Training-only penalty}
For comparison, we also evaluated a version of the penalty that uses only
training-data information, replacing $\widehat\pi_{\mathrm{future}}$ with the
estimated transition probability $\widehat p_{01}$ obtained from the training
window. This training-only penalty showed no significant correlation with the
realized gap: the window-clustered Spearman correlation was $\rho = 0.084$ with a
$95\%$ confidence interval containing zero (see Appendix~\ref{tab:diagnostics}
for detailed diagnostics). This finding highlights that while regime composition
mismatch is a real phenomenon that explains generalization gaps ex post,
estimating it before deployment remains challenging when regime transitions are
rare and training windows are limited.

\begin{table}[H]
\centering
\caption{Summary of validation results. The synthetic correlation is pooled over
configurations and models; the real-data correlation uses the window-clustered
bootstrap (whole windows resampled with replacement). $\rho$ is the Spearman
correlation between the predicted penalty and the magnitude of the realized regime
gap $|R_{\mathrm{future}}-\Rmix|$; the identity row reports the Pearson correlation
between the realized gap and the Lemma~\ref{lem:gap} prediction
$(\widehat\pi_{\mathrm{future}}-\widehat\pi)(R_1-R_0)$.}
\label{tab:results}
\small
\begin{tabular}{lcccc}
\toprule
Setting & $\rho$ (penalty vs.\ $|\text{gap}|$) & $95\%$ CI & Units & $n$ \\
\midrule
Synthetic (pooled)              & $0.716$ & ---              & config $\times$ model & $400$ \\
Real, all indices (clustered)   & $0.729$ & $[0.635,\,0.801]$ & independent windows   & $84$ \\
Real, all indices (naive pool)  & $0.729$ & not independent  & model $\times$ window & $420$ \\
\midrule
\multicolumn{2}{l}{Lemma~\ref{lem:gap} identity (Pearson $r$)} & $1.000$ & exact (MAD $0.000$) & \\
\bottomrule
\end{tabular}
\end{table}

\begin{table}[H]
\centering
\caption{Per-index window-clustered Spearman correlation between the predicted
penalty and $|R_{\mathrm{future}}-\Rmix|$, with $95\%$ bootstrap intervals and the
number of independent rolling-origin windows. The point estimates are positive for
all ten markets; intervals are wide because each index contributes few windows, so
inference is strongest in aggregate (Table~\ref{tab:results}).}
\label{tab:perindex}
\small
\begin{tabular}{lccc}
\toprule
Index & $\rho$ & $95\%$ CI & Windows \\
\midrule
S\&P 500      & $0.766$ & $[\phantom{-}0.312,\,0.890]$ & $9$ \\
NASDAQ        & $0.649$ & $[\phantom{-}0.179,\,0.869]$ & $8$ \\
Dow Jones     & $0.372$ & $[-0.267,\,0.758]$          & $9$ \\
Russell 2000  & $0.711$ & $[\phantom{-}0.237,\,0.872]$ & $8$ \\
FTSE 100      & $0.670$ & $[\phantom{-}0.109,\,0.850]$ & $9$ \\
DAX           & $0.551$ & $[-0.156,\,0.873]$          & $9$ \\
CAC 40        & $0.659$ & $[\phantom{-}0.233,\,0.849]$ & $9$ \\
Nikkei 225    & $0.519$ & $[\phantom{-}0.241,\,0.703]$ & $9$ \\
Hang Seng     & $0.607$ & $[\phantom{-}0.062,\,0.854]$ & $9$ \\
EURO STOXX 50 & $0.949$ & $[\phantom{-}0.700,\,0.957]$ & $5$ \\
\midrule
All (clustered) & $0.729$ & $[\phantom{-}0.635,\,0.801]$ & $84$ \\
\bottomrule
\end{tabular}
\end{table}

\begin{figure}[H]
\centering
\includegraphics[width=0.7\textwidth]{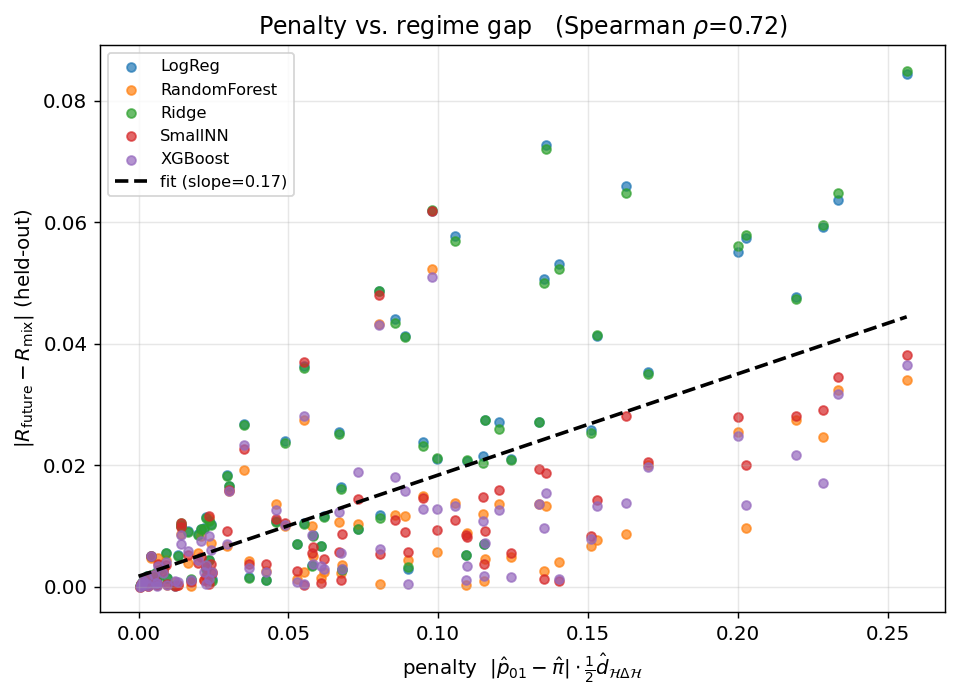}
\caption{Synthetic sweep: predicted penalty versus the magnitude of the realized
regime gap, with points colored by model and a least-squares trend line. The
positive slope ($\rho=0.716$ pooled) holds within every model family.}
\label{fig:synthetic}
\end{figure}

\begin{figure}[H]
\centering
\includegraphics[width=0.6\textwidth]{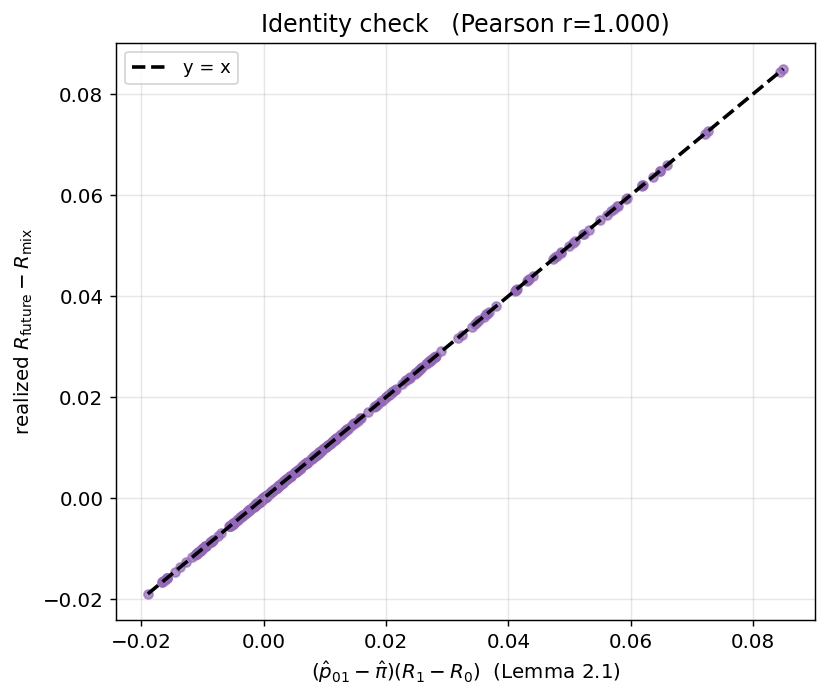}
\caption{Identity check: realized $R_{\mathrm{future}}-\Rmix$ versus the
Lemma~\ref{lem:gap} prediction $(\widehat\pi_{\mathrm{future}}-\widehat\pi)(R_1-R_0)$.
Points lie on the line $y=x$ (Pearson $r=1.000$), confirming the decomposition is
exact.}
\label{fig:identity}
\end{figure}

\begin{figure}[H]
\centering
\includegraphics[width=0.7\textwidth]{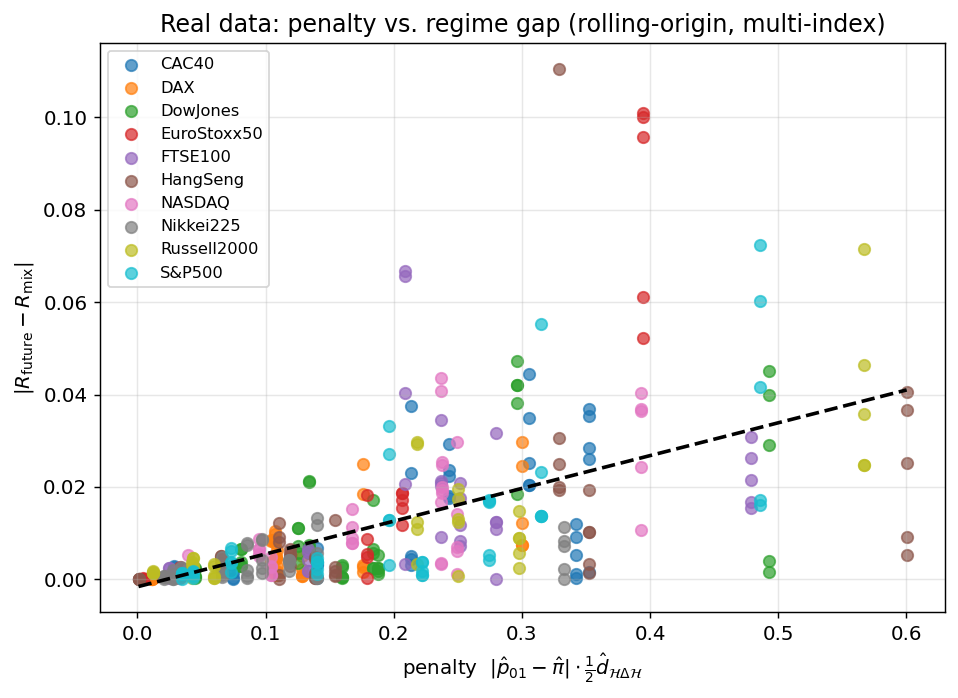}
\caption{Real data (rolling-origin, ten indices): predicted penalty versus the
magnitude of the realized regime gap, one set of points per index. The clustered
correlation is $\rho=0.729$ ($95\%$ CI $[0.635,0.801]$, $84$ windows).}
\label{fig:real}
\end{figure}

\section{Discussion}\label{sec:discussion}

The standard response to a model that fails after deployment is to blame the
model---the architecture was too complex, the regularization was insufficient,
the features were not invariant enough \citep{bousquet2004,zhang2017}. All of these
diagnoses point in the same direction: reduce $\Delta$, the regime sensitivity
term, through better algorithmic design. This paper proves that this diagnosis
is structurally incomplete, and in the worst case, structurally irrelevant.

We establish three results that together shift the burden of explanation from
the model to the environment. First, Lemma~\ref{lem:gap} provides an exact
algebraic decomposition showing that the generalization gap factorizes as
$(p_{01} - \pi) \cdot (R_1 - R_0)$. The model enters only through the second
factor; the first factor is a property of the world. Second, we identify an
explicit escape condition: when the training crisis fraction $\pi$ matches the
true transition probability $p_{01}$, the penalty vanishes identically. Third,
and most critically, we demonstrate that this escape condition is empirically
unreachable. On real equity data, a domain classifier separates calm and crisis
regimes with near-perfect accuracy ($\frac{1}{2}\hat{d}_{\mathcal{H}\Delta\mathcal{H}}
= 0.93 \pm 0.02$), confirming that the feature geometry of crises is highly
distinguishable \citep{bendavid2010,kifer2004}. Yet estimating $p_{01}$ from
training windows and plugging it into the penalty yields no significant
correlation with actual deployment gaps ($\rho \approx 0.084$, 95\% CI contains
zero). The two numbers together tell the story: regimes are obvious in
hindsight, invisible in foresight, and the resulting gap is not the model's
fault.

Theorem~\ref{thm:lower} sharpens this point theoretically. Under pure
calm-state training ($\pi = 0$), we prove a Le Cam minimax lower bound showing
that excess risk of at least $p_{01} \cdot \frac{1}{2}d_{\mathcal{H}\Delta\mathcal{H}}$
is mathematically irreducible, regardless of sample size
\citep{lecam1986,yu1997assouad}. No amount of calm-state data can close this gap.
Under mixed training ($\pi > 0$), the lower bound relaxes to
$\Theta(1/\sqrt{n_{\text{eff}}})$, meaning the learner gradually identifies the
regime structure as crisis samples accumulate. But the practical problem
remains: without a reliable estimate of $p_{01}$, one cannot know whether the
training composition $\pi$ is correctly calibrated. The escape condition exists
in theory; the world denies it in practice.

This reframes the challenge of deployment under non-stationarity. Improving
generalization is not exclusively a machine learning architecture problem. The
internal component $\Delta = \frac{1}{2}d_{\mathcal{H}\Delta\mathcal{H}}$ is
within the algorithm's control—modifiable via regularization, domain-invariant
representations, or invariant risk minimization \citep{arjovsky2019,krueger2021}.
The external component $|p_{01} - \pi|$ depends on the timing of regime
transitions and lies outside the algorithm's direct reach
\citep{hamilton1989}. Optimization can mitigate the impact of a regime shift,
but it cannot predict its arrival. Deployment risk under regime-switching
dynamics is therefore bounded not by nominal sample size, but by the physical
scarcity of historical macro-state transitions.

Consequently, we do not present this framework as an ex ante forecasting tool.
We present it as a mathematically rigorous diagnostic instrument. When a model
fails, the framework answers a specific question: was the failure due to regime
mismatch, or to something else? The penalty computed using the realized future
crisis fraction tracks actual train-to-deployment gaps with Spearman
$\rho = 0.729$ (95\% CI $[0.635, 0.801]$), confirming that the mechanism is
real and the decomposition is exact. This enables post-hoc failure auditing,
analogous to Value-at-Risk frameworks in banking \citep{jorion2007}: one can
stress-test deployment scenarios under hypothetical transition probabilities
and isolate the structural component of risk that no amount of model refinement
can eliminate.

Several limitations define the boundaries of the current framework. The
two-state Markov assumption provides analytical tractability but may not
capture the full complexity of real-world regime dynamics
\citep{hamilton1989,ang2002}. The framework is diagnostic rather than
prescriptive: it identifies where regime risk emerges, but does not guarantee
that modifying training composition will recover lost performance in all
settings. Extending the framework to multiple regimes, continuous state
strength, and prescriptive training strategies remains important future work.
Validating the diagnostic utility in other domains where regime transitions are
consequential---healthcare, cybersecurity, autonomous systems---will further
test the generality of the decomposition \citep{quinonero2008}.

\section{Conclusion}\label{sec:conclusion}

We introduced a regime-aware generalization framework for Markov-switching
distribution shifts. The central contribution is an exact decomposition of
future deployment risk into a standard empirical-complexity term and a
regime-mismatch penalty $|p_{01} - \pi| \cdot \frac{1}{2}d_{\mathcal{H}\Delta\mathcal{H}}$,
extended to $\beta$-mixing data via a spectral-gap-adjusted effective sample
size. A matching Le Cam lower bound proves that under pure calm-state training,
this penalty is irreducible. Empirical validation on ten global equity indices
confirms that the penalty mechanism is real: when computed with the realized
future crisis fraction, it tracks deployment gaps tightly ($\rho = 0.729$).
However, the training-data-only estimator fails completely ($\rho \approx 0.084$),
despite near-perfect regime separability ($\frac{1}{2}\hat{d} = 0.93 \pm 0.02$).
This negative result is itself a finding: deployment risk is governed by the
physical scarcity of regime transitions, not by nominal sample size, and the
escape condition $\pi = p_{01}$---while mathematically identified---is
empirically unreachable from historical windows alone. These results establish
that statistical learning safety under non-stationarity is fundamentally a
regime-timing problem ($p_{01}$) rather than solely an algorithmic optimization
problem ($\Delta \to 0$), and that rigorous post-hoc auditing constitutes a
necessary complement to predictive generalization bounds.
\bibliographystyle{plainnat}
\bibliography{references}

@article{bendavid2010,
  author    = {Shai Ben-David and John Blitzer and Koby Crammer and
               Alex Kulesza and Fernando Pereira and Jennifer Wortman Vaughan},
  title     = {A theory of learning from different domains},
  journal   = {Machine Learning},
  volume    = {79},
  number    = {1--2},
  pages     = {151--175},
  year      = {2010},
  doi       = {https://doi.org/10.1007/s10994-009-5152-4}
}

@inproceedings{mansour2009,
  author    = {Yishay Mansour and Mehryar Mohri and Afshin Rostamizadeh},
  title     = {Domain Adaptation: Learning Bounds and Algorithms},
  booktitle = {Proceedings of the 22nd Annual Conference on Learning Theory
               ({COLT} 2009)},
  year      = {2009},
  address   = {Montr\'{e}al, Canada}
}

@book{vapnik1998,
title={Statistical Learning Theory},
author={Vapnik, Vladimir N.},
year={1998},
publisher={Wiley},
isbn={9780471030034}
}

@incollection{bousquet2004,
title={Introduction to Statistical Learning Theory},
author={Bousquet, Olivier and Boucheron, St{'e}phane and Lugosi, G{'a}bor},
booktitle={Advanced Lectures on Machine Learning},
pages={169--207},
year={2004},
publisher={Springer},
doi={https://doi.org/10.1007/978-3-540-28650-9_8}
}

@book{quinonero2008,
title={Dataset Shift in Machine Learning},
author={Quionero-Candela, Joaquin and Sugiyama, Masashi and Schwaighofer, Anton and Lawrence, Neil D.},
year={2008},
publisher={MIT Press},
isbn={9780262170055}
}

@article{hamilton1989,
title={A New Approach to the Economic Analysis of Nonstationary Time Series and the Business Cycle},
author={Hamilton, James D.},
journal={Econometrica},
volume={57},
number={2},
pages={357--384},
year={1989},
doi={https://doi.org/10.2307/1912559}
}

@article{yu1994,
  author    = {Bin Yu},
  title     = {Rates of Convergence for Empirical Processes of Stationary
               Mixing Sequences},
  journal   = {The Annals of Probability},
  volume    = {22},
  number    = {1},
  pages     = {94--116},
  year      = {1994},
  doi       = {http://www.jstor.org/stable/2244496}
}

@article{davydov1973,
  author    = {Yu.~A. Davydov},
  title     = {Mixing conditions for {M}arkov chains},
  journal   = {Theory of Probability and its Applications},
  volume    = {18},
  number    = {2},
  pages     = {312--328},
  year      = {1973},
  doi       = {https://doi.org/10.1137/1118033}
}

@article{bartlett2002,
  author    = {Peter L. Bartlett and Shahar Mendelson},
  title     = {Rademacher and {G}aussian Complexities: Risk Bounds and
               Structural Results},
  journal   = {Journal of Machine Learning Research},
  volume    = {3},
  pages     = {463--482},
  year      = {2002}
}

@book{mohri2018,
  author    = {Mehryar Mohri and Afshin Rostamizadeh and Ameet Talwalkar},
  title     = {Foundations of Machine Learning},
  edition   = {2nd},
  publisher = {The {MIT} Press},
  address   = {Cambridge, {MA}},
  year      = {2018}
}

@inproceedings{kifer2004,
  author    = {Daniel Kifer and Shai Ben-David and Johannes Gehrke},
  title     = {Detecting Change in Data Streams},
  booktitle = {Proceedings of the 30th International Conference on Very Large
               Data Bases ({VLDB} 2004)},
  pages     = {180--191},
  year      = {2004},
  address   = {Toronto, Canada},
  url       = {http://www.vldb.org/conf/2004/RS5P1.PDF},
  doi       = {10.1016/B978-012088469-8.50019-X}
}

@book{lecam1986,
  author    = {Lucien Le Cam},
  title     = {Asymptotic Methods in Statistical Decision Theory},
  series    = {Springer Series in Statistics},
  publisher = {Springer-Verlag},
  address   = {New York},
  year      = {1986}
}

@incollection{yu1997assouad,
  author    = {Bin Yu},
  title     = {Assouad, {F}ano, and {L}e {C}am},
  booktitle = {Festschrift for Lucien Le Cam},
  editor    = {David Pollard and Erik Torgersen and Grace L. Yang},
  pages     = {423--435},
  publisher = {Springer},
  address   = {New York},
  year      = {1997},
  doi       = {10.1007/978-1-4612-1880-7_29}
}

@article{suarez2024,
  title={Machine Learning for Financial Prediction Under Regime Change Using Technical Analysis: A Systematic Review},
  author={Su{\'a}rez Cetrulo, Andr{\'e}s L. and Quintana, David and Cervantes, Alejandro},
  journal={International Journal of Interactive Multimedia and Artificial Intelligence},
  volume={9},
  number={1},
  pages={137--148},
  year={2024},
  doi={10.9781/ijimai.2023.06.003}
}

@article{shu2024,
  title={Dynamic Asset Allocation with Asset-Specific Regime Forecasts},
  author={Shu, Yizhan and Yu, Chenyu and Mulvey, John M.},
  journal={Annals of Operations Research},
  volume={346},
  pages={285--318},
  year={2024},
  doi={10.1007/s10479-024-06266-0}
}

@inproceedings{xiang2024,
  title={{RSAP-DFM}: Regime-Shifting Adaptive Posterior Dynamic Factor Model for Stock Returns Prediction},
  author={Xiang, Quanzhou and Chen, Zhan and Sun, Qi and Jiang, Rujun},
  booktitle={Proceedings of the Thirty-Third International Joint Conference on Artificial Intelligence},
  pages={6116--6124},
  year={2024},
  doi={10.24963/ijcai.2024/676}
}

@article{zaremba2024,
    title={What drives stock returns across countries? Insights from machine learning models},
    author={Zaremba, Adam and Cakici, Nusret},
    journal={International Review of Financial Analysis},
    volume={96},
    pages={103576},
    year={2024},
    issn={1057-5219},
    doi={10.1016/j.irfa.2024.103576},
    publisher={Elsevier}
}

@article{staehr2024,
  author    = {S{\o}ren Staehr and others},
  title     = {Forecasting stock returns with regime-switching models},
  journal   = {Journal of Financial Economics},
  year      = {2024},
  note      = {Working paper}
}

@inproceedings{zhang2017,
  author    = {Zhang, Chiyuan and Bengio, Samy and Hardt, Moritz and Recht, Benjamin and Vinyals, Oriol},
  title     = {Understanding Deep Learning Requires Rethinking Generalization},
  booktitle = {International Conference on Learning Representations (ICLR)},
  year      = {2017},
  doi       = {10.48550/arXiv.1611.03530}
}

@inproceedings{arjovsky2019,
  author    = {Arjovsky, Martin and Bottou, L{\'e}on and Gulrajani, Ishaan and Lopez-Paz, David},
  title     = {Invariant Risk Minimization},
  booktitle = {arXiv preprint},
  year      = {2019},
  doi       = {
https://doi.org/10.48550/arXiv.1907.02893}
}

@article{krueger2021,
  author    = {Krueger, David and Caballero, Ethan and Jacobsen, Joern-Henrik and Zhang, Amy and Binas, Jonathan and Zhang, Dinghuai and Le Priol, Remi and Courville, Aaron},
  title     = {Out-of-Distribution Generalization via Risk Extrapolation},
  journal   = {International Conference on Machine Learning (ICML)},
  year      = {2021},
  doi       = {10.48550/arXiv.2003.00688}
}

@book{jorion2007,
  author    = {Jorion, Philippe},
  title     = {Value at Risk: The New Benchmark for Managing Financial Risk},
  edition   = {3rd},
  year      = {2007},
  publisher = {McGraw-Hill},
  isbn      = {978-0071464956}
}

@article{ang2002,
  author    = {Ang, Andrew and Bekaert, Geert},
  title     = {International Asset Allocation with Regime Shifts},
  journal   = {Review of Financial Studies},
  volume    = {15},
  number    = {4},
  pages     = {1137--1187},
  year      = {2002},
  doi       = {https://doi.org/10.1093/rfs/15.4.1137}
}
\section*{Acknowledgments}
The author thanks AI for assistance with language polishing
and LaTeX formatting. All intellectual content is solely the author's own.

\appendix
\section{Proofs}\label{app:proofs}

We prove the results that are original to this work or essential to the main
bound. Throughout, $\ell\in[0,1]$ and all expectations are under the stated
distributions. Two standard facts are used without proof:
\begin{itemize}[leftmargin=2em]
\item[(F1)] $\half\dHH(Q,R)\le\dTV(Q,R)$, since every disagreement set
$\{x:h(x)\neq h'(x)\}$ is an event and the supremum over such sets is dominated by
the supremum over all events (Definition~\ref{def:hdh}).
\item[(F2)] The effective-sample-size bound $\neff\ge ng/(\ln(n\Cmu)+2)$ and the
crisis-fraction concentration $|\widehat\pi-\pi|\le\eta_\pi$ are direct
applications of the mixing inequalities of \citet{yu1994} to our two-state chain,
with block length $b=\lceil\ln(n\Cmu)/g\rceil$ chosen so that $\beta(b)\le1/n$.
\end{itemize}

\subsection{Proof of Lemma~\ref{lem:gap} (future-mix identity)}
\begin{proof}
\textbf{Step 1 (expand each risk).} By Definition~\ref{def:trainfuture} and
linearity of expectation,
\[
  \Rfut(f)=p_{00}R_0(f)+p_{01}R_1(f),
  \qquad
  \Rmix(f)=(1-\pi)R_0(f)+\pi R_1(f).
\]

\textbf{Step 2 (subtract).}
\[
  \Rfut(f)-\Rmix(f)
  =\bigl[p_{00}-(1-\pi)\bigr]R_0(f)+\bigl[p_{01}-\pi\bigr]R_1(f).
\]

\textbf{Step 3 (use $p_{00}=1-p_{01}$).} The coefficient of $R_0(f)$ becomes
$(1-p_{01})-(1-\pi)=\pi-p_{01}$, so
\[
  \Rfut(f)-\Rmix(f)=(\pi-p_{01})R_0(f)+(p_{01}-\pi)R_1(f).
\]

\textbf{Step 4 (factor).} Since $\pi-p_{01}=-(p_{01}-\pi)$,
\[
  \Rfut(f)-\Rmix(f)=(p_{01}-\pi)\bigl(R_1(f)-R_0(f)\bigr).\qedhere
\]
\end{proof}

\subsection{Proof of Corollary~\ref{cor:shift} (regime-shift inequality)}
\begin{proof}
\textbf{Step 1 (bound the gap in absolute value).} From Lemma~\ref{lem:gap} and
$|ab|=|a|\,|b|$,
\[
  \Rfut(f)\le\Rmix(f)+|p_{01}-\pi|\,\bigl|R_1(f)-R_0(f)\bigr|.
\]

\textbf{Step 2 (relate the regime risks via a reference hypothesis).} For the
$0/1$ loss write $R_i(f)=\Prob_{x\sim P_i}[f(x)\neq y]$, and let
$h^\star=\arg\min_{h\in\Hc}[R_0(h)+R_1(h)]$. The triangle inequality for
classification disagreement (for $\{0,1\}$-valued $a,b,c$,
$\Prob[a\neq c]\le\Prob[a\neq b]+\Prob[b\neq c]$) gives
\[
  R_1(f)\le R_0(f)
  +\bigl|\Prob_{P_1}[f\neq h^\star]-\Prob_{P_0}[f\neq h^\star]\bigr|
  +\lambda_{01}.
\]

\textbf{Step 3 (bound the middle term by the divergence).} Since
$\Fc\subseteq\Hc$, the disagreement $f\oplus h^\star\in\Hc\Delta\Hc$, hence
\[
  \bigl|\Prob_{P_1}[f\neq h^\star]-\Prob_{P_0}[f\neq h^\star]\bigr|
  \le\sup_{h,h'\in\Hc}\bigl|\Prob_{P_1}[h\neq h']-\Prob_{P_0}[h\neq h']\bigr|
  =\half\dHH(P_1,P_0).
\]

\textbf{Step 4 (symmetrize).} Repeating Steps 2 and 3 with the roles of $P_0,P_1$
exchanged (the term $\lambda_{01}$ is symmetric) yields the two-sided bound
\[
  \bigl|R_1(f)-R_0(f)\bigr|\le\half\dHH(P_1,P_0)+\lambda_{01}.
\]

\textbf{Step 5 (combine).} Substituting Step 4 into Step 1,
\[
  \Rfut(f)\le\Rmix(f)+|p_{01}-\pi|\bigl(\half\dHH(P_1,P_0)+\lambda_{01}\bigr),
\]
and $\lambda_{01}=0$ gives the realizable form.
\end{proof}

\subsection{Proof of Theorem~\ref{thm:dhh} (mixing-aware convergence)}
\begin{proof}
\textbf{Step 1 (i.i.d.\ base bound).} For independent samples of size $m'$ per
domain and a discriminator class of VC dimension $D$ \citep{vapnik1998}, uniform VC deviation
\citep{kifer2004,bendavid2010} gives, with probability $\ge1-\delta$,
\[
  \dHH(P_0,P_1)\le\dhatHH(U_0,U_1)
  +4\sqrt{\frac{D\log(2m')+\log(2/\delta)}{m'}}.
\]
Apply this with $D=\mathrm{VCdim}(\Hc\Delta\Hc)$.\footnote{A standard result 
\citep[Lemma~3.5]{mohri2018} gives $\mathrm{VCdim}(\Hc\Delta\Hc) \leq 2d \log_2(2d) = O(d\log d)$. 
For simplicity, we state the bound with $2d$; the logarithmic factor does not affect the asymptotic rate.} 
For readability, we use the conservative simplification $D \leq 2d$ (up to the logarithmic factor).

\textbf{Step 2 (pay for dependence by blocking; the novel step).} The $2m'$
feature points are $\beta$-mixing \citep{yu1994}. Partition each domain sample into blocks of
length $b=\lceil\ln(m'\Cmu)/g\rceil$, so that $\beta(b)\le1/m'$; the
\[
  \meff\ \ge\ \frac{m'g}{\ln(m'\Cmu)+2}
\]
blocks act as approximately independent draws. The bound of Step~1 then holds with
$m'$ replaced by $\meff$, the extra $\beta(b)\le1/m'$ absorbed into constants.

\textbf{Step 3 (substitute and halve).} Using
$1/\meff\le(\ln(m'\Cmu)+2)/(m'g)$ in the Step-1 root and dividing by $2$,
\[
  \half\dHH(P_0,P_1)\le\half\dhatHH(U_0,U_1)
  +\underbrace{2\sqrt{\frac{\bigl(2d\log(2m')+\log(2/\delta)\bigr)\bigl(\ln(m'\Cmu)+2\bigr)}{m'g}}}_{=:~\eta_d}.
  \qedhere
\]
\end{proof}

\subsection{Proof of Theorem~\ref{thm:main} (main bound)}
\begin{proof}
\textbf{Step 1 (future to training population).} By Corollary~\ref{cor:shift},
\[
  \Rfut(f)\le\Rmix(f)+|p_{01}-\pi|\bigl(\half\dHH(P_1,P_0)+\lambda_{01}\bigr).
\]

\textbf{Step 2 (population to sample).} The Rademacher bound for stationary
$\beta$-mixing sequences \citep{mohri2018,yu1994} gives, with probability $\ge1-\delta$ and uniformly in
$f$,
\[
  \Rmix(f)\le\Rhat_S(f)+2\Rhat_S(\Lc_\Fc)+3\sqrt{\frac{\ln(2/\delta)}{2\neff}}.
\]

\textbf{Step 3 (explicit constant).} By (F2) and
Corollary~\ref{cor:complexity},
\[
  3\sqrt{\frac{\ln(2/\delta)}{2\neff}}\le\Lambda(n,\delta).
\]

\textbf{Step 4 (chain).} Combining Steps 1 through 3,
\[
  \Rfut(f)\le\Rhat_S(f)+2\Rhat_S(\Lc_\Fc)+\Lambda(n,\delta)
  +|p_{01}-\pi|\bigl(\half\dHH(P_1,P_0)+\lambda_{01}\bigr).
\]
The only stochastic step is Step~2, so the bound holds with probability
$\ge1-\delta$.
\end{proof}

\subsection{Proof of Theorem~\ref{thm:lower} (lower bound at $\pi=0$)}
\begin{proof}
\textbf{Step 1 (construction).} Let the feature space be $\{x^\star\}\cup C$ with
\[
  \Prob_{P_0}(x^\star)=0,\qquad \Prob_{P_1}(x^\star)=2\rho,
\]
the remaining mass on $C$. Labels on $C$ are deterministic and identical across
worlds; on $x^\star$ the label is $\mathrm{Bernoulli}(\tfrac12\pm\beta)$ in worlds
$a,b$. Let $\Hc$ contain $h_+,h_-$ that agree on $C$ and disagree on $x^\star$, so
$h_+\oplus h_-=\one[x=x^\star]$.

\textbf{Step 2 (the construction realizes $\half\dHH=\rho$).} The only detectable
distinguishing set is $\{x^\star\}$, with frequency $2\rho$ under $P_1$ and $0$
under $P_0$; hence
\[
  \half\dHH(P_1,P_0)=\tfrac12\,|2\rho-0|=\rho.
\]

\textbf{Step 3 (future excess on $x^\star$).} Under $\Pfut$ the point $x^\star$
carries mass $p_{01}\cdot 2\rho$. For a learner predicting $1$ on $x^\star$ with
probability $q$, the per-unit-mass excess over the world-optimal rule is
$2\beta(1-q)$ in world $a$ and $2\beta q$ in world $b$.

\textbf{Step 4 (the case $\pi=0$).} When $\pi=0$, training never contains
$x^\star$, so the two worlds induce identical training laws ($\TV=0$): the learner
cannot distinguish them. Taking deterministic opposite labels
($\beta\to\tfrac12$) and the minimax choice $q=\tfrac12$ \citep{lecam1986,yu1997assouad},
\[
  \max_{w\in\{a,b\}}\E[\text{excess}]
  \ \ge\ p_{01}(2\rho)\cdot\tfrac12
  \ =\ p_{01}\,\rho
  \ =\ p_{01}\cdot\half\dHH(P_1,P_0),
\]
uniformly in $n$.

\textbf{Step 5 (the case $\pi>0$).} The expected number of revealing
$x^\star$-points in training is $\neff\,\pi\cdot 2\rho$, so by Pinsker's
inequality $\TV\le 2\beta\sqrt{2\neff\pi\rho}$. Optimizing $\beta$ subject to
$\TV\le\tfrac12$ gives a worst-case excess of order
\[
  p_{01}\sqrt{\frac{\rho}{\neff\,\pi}}\ =\ \Theta\!\bigl(1/\sqrt{\neff}\bigr),
\]
which vanishes with $n$ and is therefore not matched by an $n$-uniform constant.
\end{proof}

\subsection{Proof of Proposition~\ref{prop:cert} (certification cost)}
\begin{proof}
\textbf{Step 1 (a valid certificate dominates the worst future).} Valid certificate $U(S)$ should must hold against all possible future market condition. This means, there is no certainty of future regime,  so $U(S)$ covers the least favorable $(p_{01})$ case in the confidence set.Using
$\inf_g\Rfut(g)\le\Rhat_S(f)+\Lambda(n,\delta)$ and dropping the nonnegative
$2\Rhat_S(\Lc_\Fc)$ and $\lambda_{01}$, the residual satisfies
\[
  U(S)-\Rhat_S(f)-\Lambda(n,\delta)\ \ge\ |p_{01}-\pi|\,\Delta.
\]

\textbf{Step 2 (replace $\pi$ by its estimate).} By the reverse triangle
inequality and the concentration bound (F2),
\[
  |p_{01}-\pi|\ \ge\ |p_{01}-\widehat\pi|-|\widehat\pi-\pi|
  \ \ge\ |p_{01}-\widehat\pi|-\eta_\pi.
\]
Combining the two steps gives the claim. When $\pi=p_{01}$, Lemma~\ref{lem:gap}
makes the realized gap exactly zero, so the bound is a property of certificates,
not of realized risk.
\end{proof}

\subsection{Diagnostic: Domain Classifier and Training-Only Penalty}
\label{sec:diagnostics}
\begin{table}[H]
\centering
\caption{Domain classifier performance and training-only penalty correlation}
\label{tab:diagnostics}
\small
\begin{tabular}{lccc}
\toprule
Setting & $\frac{1}{2}\hat{d}_{\mathcal{H}\Delta\mathcal{H}}$ (mean $\pm$ std) & $\rho(|\hat{p}_{01}-\hat{\pi}|\cdot\frac{1}{2}\hat{d},\ |\text{gap}|)$ & 95\% CI \\
\midrule
Synthetic & $0.60 \pm 0.23$ & $0.716$ & $[0.635,\,0.801]$ \\
Real data (HMM regimes) & $0.93 \pm 0.02$ & $0.084$ & $[-0.097,\,0.272]$ \\
\bottomrule
\end{tabular}
\end{table}
\footnote{The synthetic penalty has a strong correlation of (0.716), this confirms the mechanism works under ideal conditions. Real data achieved almost perfect regime separation (0.93) but the correlation is near zero (0.084), proving that $p_{01}$ estimation is the fundamental bottleneck rather than regime detection.}

\end{document}